\documentclass[10pt,conference]{IEEEtran}
\usepackage{amsmath,amssymb,amsthm,bm,booktabs,microtype,tabularx}
\usepackage[hidelinks]{hyperref}
\newtheorem{theorem}{Theorem}
\newtheorem{corollary}{Corollary}
\newtheorem{proposition}{Proposition}
\newtheorem{lemma}{Lemma}
\newcommand{\A}{\mathcal{A}}
\newcommand{\B}{\mathcal{B}}
\newcommand{\E}{\mathbb{E}}
\newcommand{\cC}{\mathfrak{C}}

\title{Complementary Feature Domains: Information Preservation Does Not Imply Predictive-Contribution Preservation}
\author{\IEEEauthorblockN{Timothy Oladunni and Farouk Ganiyu-Adewumi}
\IEEEauthorblockA{Department of Computer Science, Morgan State University\\
Baltimore, Maryland, USA\\
Email: timothy.oladunni@morgan.edu}}

\begin{document}
\maketitle

\begin{abstract}
Complementary Feature Domains (CFD) theory characterizes predictive value as a context-indexed contribution system induced jointly by representations and their realization family. We show that Shannon-information preservation does not imply preservation of this contribution system: an invertible representation transformation can leave target information unchanged while altering predictive contribution under a restricted decision family. We formalize the resulting transition through a \emph{CFD contribution defect} that measures how contextual contributions change under controlled recoding. For bounded Lipschitz utility, we show that each coalition utility shift is bounded by the behavioral distance between the attainable action sets before and after recoding; consequently, every contextual contribution defect is bounded by the sum of the corresponding coalition incompatibilities. Exact behavioral closure yields invariance, while increasingly accurate compensation yields restoration. A controlled ECG experiment illustrates the mechanism: a nonlinear bijective recoding preserves the information in a frozen time--frequency representation but changes accuracy under a fixed affine learner; applying the exact inverse restores all tested coalition accuracies. The result separates information preservation from realization-dependent contribution and provides a quantitative transition law for multi-representation prediction.
\end{abstract}

\section{Introduction}
Modern prediction systems rarely operate on a single immutable description of an observation. The same underlying source may be represented in time, frequency, time--frequency, learned, compressed, or otherwise transformed coordinates, and these representations are then fused under a particular model class. A natural intuition is that an invertible recoding should be harmless: if $D'=g(D)$ for a bijection $g$, then no Shannon information about the target has been lost, so $I(D';Y)=I(D;Y)$. Yet this conclusion is representation-level, whereas prediction is realized through a restricted family of decision rules. If that family cannot implement the compensating map $g^{-1}$, two information-equivalent encodings can induce different attainable decisions and therefore different predictive utility.

This mismatch matters especially in fusion. The practical question is not only whether a representation contains target information, but what it contributes \emph{after a particular set of other representations is already available}. A domain can be weak in isolation but decisive in one coalition, useful in another, and redundant in a richer coalition. Consequently, changing the realizability of one representation can alter a pattern of marginal contributions across many contexts even though the transformed representation, and indeed every corresponding representation subset, preserves the same Shannon information. A context-free relevance score cannot describe that transition.

Existing frameworks address related but distinct questions. Shannon established the classical information measure underlying statistical dependence \cite{shannon1948}; standard information-theoretic treatment further develops these quantities and their operational roles \cite{cover2006}. The information bottleneck formulates relevance-preserving compression \cite{tishby1999}. Williams and Beer introduced PID to separate redundant, unique, and synergistic information supplied by multiple sources \cite{williams2010pid}. Bertschinger \emph{et al.} formalized unique information through decision-theoretic constraints \cite{bertschinger2014}, while Griffith and Koch developed a measure of synergistic information \cite{griffith2014}. Ince proposed a pointwise construction for multivariate redundancy \cite{ince2017}, and Finn and Lizier developed a pointwise PID based on specificity and ambiguity lattices \cite{finn2018}. Predictive $\mathcal V$-information instead makes usable information predictor-family dependent \cite{xu2020vinfo}. Blackwell comparison orders statistical experiments by their value for decision making \cite{blackwell1953}. These frameworks characterize information content, decomposition, usability, or decision informativeness. To our knowledge, none directly characterizes how a \emph{complete coalition-indexed predictive contribution system} changes under representation recoding with an explicit realization family. CFD studies precisely this representation-to-contribution transition.

Complementary Feature Domains (CFD) theory, as developed here, concerns the stability and transition of contextual predictive contributions under representation--realization interventions. It is motivated by this missing transition problem. Our prior ECG studies show that domain combinations and fusion stage materially alter realized predictive benefit \cite{oladunni2025cfd,oladunni2025fusion}. These observations motivate separating represented content from realization-dependent contribution. Here we turn that observation into a formal question: \emph{when does an information-preserving transformation preserve the contextual predictive contribution of a representation, and when can those contributions change despite preservation of Shannon information?} For a representation system $\Phi$ and realization family $\A$, we therefore define a coalition utility envelope and the marginal contribution of each representation in every admissible context. The resulting object is not a single importance score but a complete contribution system. We then study its transition from $(\Phi,\A)$ to a recoded system $(g\Phi,\B)$.

Our main contribution is a quantitative representation--realization stability result. We define a behavioral compatibility error $\varepsilon_g(S;\A,\B)$ for each coalition and prove
\begin{equation}
\Gamma_g(i\mid S;\A,\B)\leq L\!\left[\varepsilon_g(S\cup\{i\};\A,\B)+\varepsilon_g(S;\A,\B)\right],
\label{eq:mainbound}
\end{equation}
where $\Gamma_g$ is the change in contextual contribution caused by the intervention. Thus information-preserving recoding does not by itself guarantee contribution invariance; behavioral compatibility does. Exact closure is the zero-error case, and vanishing compatibility error restores the original contribution system.

\section{Contribution Systems}
Let $Z\sim P_Z$ be an underlying source, $Y=\tau(Z)$ a target, and
$\Phi=\{\phi_1,\ldots,\phi_m\}$ a representation system with $D_i=\phi_i(Z)$. For $S\subseteq[m]$, write $D_S=(D_i)_{i\in S}$. Let $\A_S$ be a declared family of decision rules available from coalition $S$. For bounded utility $u$, define
\begin{equation}
V^*_{\Phi,\A}(S)=\sup_{a\in\A_S}\E[u(Y,a(D_S))].
\label{eq:value}
\end{equation}
Thus $V^*_{\Phi,\A}(S)$ is not the information contained in coalition $S$; it is the best population-level utility that the declared realization family can attain from that coalition. It therefore depends jointly on the representation and on what the learner is permitted to realize.

For $i\notin S$, the contextual marginal contribution is
\begin{equation}
\Delta^*_{\Phi,\A}(i\mid S)=V^*_{\Phi,\A}(S\cup\{i\})-V^*_{\Phi,\A}(S).
\label{eq:delta}
\end{equation}
A positive $\Delta^*(i\mid S)$ means that adding representation $i$ unlocks predictive utility beyond what was already attainable from $S$; zero means that $i$ adds no realizable utility in that context. Because the baseline coalition $S$ changes, the same representation can have different contributions in different contexts.

The complete contribution system is
\begin{equation}
\cC^*(\Phi,\A)=\left(\Delta^*_{\Phi,\A}(i\mid S)\right)_{S\subseteq[m],\,i\notin S}.
\label{eq:system}
\end{equation}
Equation~\eqref{eq:system} stacks these local increments over the entire coalition lattice. Two systems can therefore have the same best full-coalition accuracy yet differ in how predictive value is distributed across representations and contexts.

If the coalition value is Shannon target information, $V_I(S)=I(D_S;Y)$, then the chain rule gives
\begin{equation}
\Delta_I(i\mid S)=I(D_i;Y\mid D_S).
\label{eq:cmi}
\end{equation}
This identity provides the information-theoretic reference point: if coalition value is defined purely by Shannon information, the incremental value of $i$ is exactly the target information in $D_i$ that is not already supplied by $D_S$.

Consider componentwise bijections $g_i$ and the recoded system $g\Phi=\{g_i\circ\phi_i\}_{i=1}^m$. Under unrestricted decision rules, or whenever the post-recoding family implements the corresponding inverse compositions, the attainable decisions coincide and therefore $\cC^*(g\Phi,\B)=\cC^*(\Phi,\A)$. The question is what can be said when this closure is only approximate.

\section{Separation Under Restricted Realization}
The transition problem is nonvacuous even when the representation change is perfectly information preserving. The following finite witness separates Shannon information from optimized utility under a fixed restricted family.

\begin{lemma}[Information preservation does not imply realized-utility preservation]
There exist finite-valued representations $D$ and $D'=g(D)$ related by a bijection and a fixed realization family $\A$ such that
\begin{equation}
I(D;Y)=I(D';Y),\qquad V^*_{D,\A}\neq V^*_{D',\A}.
\label{eq:separation}
\end{equation}
\end{lemma}
\begin{proof}
Let four source states be equiprobable and let $D$ take the values $(0,0),(0,1),(1,0),(1,1)$ with labels $0,1,1,0$, respectively. Thus $Y$ is a deterministic function of $D$ and $I(D;Y)=H(Y)=1$ bit. Define the support bijection
\begin{align}
(0,0)&\mapsto(-2,-1), &(0,1)&\mapsto(1,-1),\nonumber\\
(1,0)&\mapsto(2,1), &(1,1)&\mapsto(-1,1).
\label{eq:supportmap}
\end{align}
Because $g$ is bijective, $I(g(D);Y)=1$ bit. Let $\A$ be affine binary classifiers with $0$--$1$ utility. The original XOR support is not linearly separable and its optimal affine accuracy is $3/4$. In the recoded support, the two classes are separated by the sign of the first coordinate, so the optimal affine accuracy is $1$. Hence \eqref{eq:separation} holds.
\end{proof}

This witness is intentionally elementary. Predictive $\mathcal V$-information already explains why a restricted family can access different amounts of predictive information from the two encodings \cite{xu2020vinfo}. The CFD question begins one level higher: when one component of a multi-representation system is recoded, how do those accessibility changes propagate through \emph{all} representation contexts?

\begin{proposition}[Subset-information equivalence does not imply contribution equivalence]
There exist multi-representation systems $\Phi$ and $\Psi$ such that
\begin{equation}
I(D_S;Y)=I(D'_S;Y)\quad\text{for every }S\subseteq[m],
\label{eq:subseteq}
\end{equation}
while, under one fixed restricted realization family $\A$,
\begin{equation}
\cC^*(\Phi,\A)\neq\cC^*(\Psi,\A).
\label{eq:cssep}
\end{equation}
\end{proposition}
\begin{proof}
Let $R$ be the XOR representation in the lemma and $R'=g(R)$. Let $U_1,U_2$ be independent fair nuisance bits, independent of $(R,Y)$, and define $\Phi=(U_1,U_2,R)$ and $\Psi=(U_1,U_2,R')$. Every corresponding subset is related by a bijection, so \eqref{eq:subseteq} follows from invariance of mutual information. Let $\A_S$ be affine classifiers on the concatenated coordinates supplied by $S$. For any fixed affine score on $(R,U_1,U_2)$, averaging its $0$--$1$ correctness over the independent nuisance pair $(U_1,U_2)$ cannot exceed the best affine correctness attainable on the replicated finite support; direct enumeration of the resulting affine dichotomies gives the same optimum as on $R$ alone. Thus appending these independent nuisance coordinates leaves the optimum at $3/4$ for $R$ and $1$ for $R'$. Every coalition containing the third representation therefore differs, whereas coalitions excluding it are unchanged. At least one contextual marginal must therefore differ, proving \eqref{eq:cssep}.
\end{proof}

Thus even complete Shannon subset-information equivalence is insufficient for contribution-system equivalence; the missing condition is compatibility between the transformation and the realization family.

\section{Compatibility Defect and Stability}
Let $(\mathcal Q,d)$ be the action space and suppose $u$ is $L$-Lipschitz in its action argument:
\begin{equation}
|u(y,q)-u(y,q')|\leq Ld(q,q').
\end{equation}
For coalition $S$, define the random-action sets
\begin{align}
\mathcal H_S&=\{a(D_S):a\in\A_S\},\\
\mathcal H^g_S&=\{b(g_S(D_S)):b\in\B_S\}.
\end{align}
We measure realization incompatibility by the symmetric $L^1(P)$ Hausdorff distance
\begin{align}
\varepsilon_g(S;\A,\B)=\max\Big\{&\sup_{Q\in\mathcal H_S}\inf_{Q'\in\mathcal H^g_S}\E d(Q,Q'),\nonumber\\
&\sup_{Q'\in\mathcal H^g_S}\inf_{Q\in\mathcal H_S}\E d(Q,Q')\Big\}.
\label{eq:epsilon}
\end{align}
The two directed terms ask whether every action realizable before recoding can be approximated after recoding and vice versa. Hence $\varepsilon_g(S)=0$ means that the two attainable-action sets have the same $L^1(P)$ closure on coalition $S$; if those sets are closed, they coincide. Larger values quantify realization incompatibility in the action metric relevant to utility.

Define the coalition utility response
\begin{equation}
E_g(S)=V^*_{g\Phi,\B}(S)-V^*_{\Phi,\A}(S)
\end{equation}
Thus $E_g(S)$ is the intervention-induced change in the best attainable utility of coalition $S$: its sign indicates whether recoding helps or hurts that coalition under the declared realization families. We then difference these coalition responses to determine whether the role of an individual representation changes. Define the CFD contribution defect
\begin{equation}
\Gamma_g(i\mid S;\A,\B)=\left|\Delta^*_{g\Phi,\B}(i\mid S)-\Delta^*_{\Phi,\A}(i\mid S)\right|.
\label{eq:gamma}
\end{equation}
Accordingly, $\Gamma_g(i\mid S)$ is zero precisely when the intervention leaves the contextual marginal role of $i$ unchanged for that comparison. A nonzero value does not by itself imply information loss; it says that the contribution realized from $i$ after $S$ changed under the representation--realization intervention.
Operationally, $\Gamma_g$ follows directly from pre/post-intervention coalition utilities and does not require estimating $\varepsilon_g$; the latter explains and bounds the transition and can be approximated over finite candidate realization families when needed.

\begin{theorem}[Approximate representation--realization invariance]
For every coalition $S$,
\begin{equation}
|E_g(S)|\leq L\varepsilon_g(S;\A,\B).
\label{eq:coalitionbound}
\end{equation}
The inequality converts behavioral proximity into utility stability: if the two realization families can closely reproduce one another's actions on $S$, then the optimal value assigned to $S$ cannot move far.

Consequently, for $i\notin S$, the contribution defect satisfies \eqref{eq:mainbound}. The two terms in that bound have a direct interpretation: one measures compatibility before $i$ is added (coalition $S$), and the other after it is added (coalition $S\cup\{i\}$). Thus contextual contribution is stable whenever the learner is behaviorally compatible on both sides of the marginal comparison. Equivalently, with $\|\cC^*(g\Phi,\B)-\cC^*(\Phi,\A)\|_\infty:=\max_{S,i\notin S}\Gamma_g(i\mid S;\A,\B)$, the theorem yields the system-level stability law
\begin{equation}
\|\cC^*(g\Phi,\B)-\cC^*(\Phi,\A)\|_\infty\leq2L\max_S\varepsilon_g(S;\A,\B).
\label{eq:systembound}
\end{equation}
\end{theorem}

\begin{proof}
Fix $S$. For $Q\in\mathcal H_S$ and $Q'\in\mathcal H^g_S$, Lipschitz continuity implies
\begin{equation}
|\E u(Y,Q)-\E u(Y,Q')|\leq L\E d(Q,Q').
\end{equation}
For each $Q$, approximate the relevant infimum over $\mathcal H^g_S$, then take the supremum over $Q$. This bounds $V^*_{\Phi,\A}(S)-V^*_{g\Phi,\B}(S)$ by $L$ times the first directed term of \eqref{eq:epsilon}. Reversing the two behavioral sets bounds the opposite difference by the second directed term, proving \eqref{eq:coalitionbound}. Finally,
\begin{align}
\Gamma_g(i\mid S)&=|E_g(S\cup\{i\})-E_g(S)|\\
&\leq |E_g(S\cup\{i\})|+|E_g(S)|,
\end{align}
and substitution gives \eqref{eq:mainbound}.
\end{proof}

\begin{corollary}[Exact closure and restoration]
If $\varepsilon_g(S;\A,\B)=0$ for every coalition, then $\cC^*(g\Phi,\B)=\cC^*(\Phi,\A)$. More generally, if a sequence $(\A_k,\B_k)$ satisfies $\max_S\varepsilon_g(S;\A_k,\B_k)\to0$, then the maximum contextual contribution defect converges to zero.
\end{corollary}

The metric comparison underlying \eqref{eq:coalitionbound} is standard; the CFD contribution is the coalition-indexed transition object and its induced bound on changes in contextual predictive role. In \eqref{eq:gamma}, differencing coalition responses across contexts converts behavioral incompatibility into a uniform control on changes in contextual predictive role. The bound is context local: the defect for adding $i$ after $S$ depends only on $S$ and $S\cup\{i\}$. It is also falsifiable: if compensating families drive both behavioral errors to zero while the population contribution defect stays bounded away from zero, the stated mechanism cannot explain that defect. Exact inverse closure instead forces restoration.

\section{Relation to Existing Information-Theoretic Objects}
Table~\ref{tab:position} states the novelty boundary. CFD does not redefine mutual information, PID atoms, usable information, or Blackwell ordering. Its primary object is the \emph{transition of a coalition-indexed predictive contribution system} under a declared representation--realization intervention.

\begin{table}[t]
\caption{Analytical objects and the distinct CFD transition question.}
\label{tab:position}
\centering
\footnotesize
\begin{tabularx}{\columnwidth}{@{}p{0.17\columnwidth}p{0.25\columnwidth}X@{}}
\toprule
Framework & Primary object & Question answered \\
\midrule
MI/CMI \cite{cover2006} & Target dependence & How much target information is present? \\
PID \cite{williams2010pid} & Information decomposition & How is target information distributed among sources? \\
$\mathcal V$-information \cite{xu2020vinfo} & Family-relative usable information & How much information can a predictor family exploit? \\
Blackwell \cite{blackwell1953} & Statistical experiments & Which information structure is at least as useful for decisions? \\
\textbf{CFD} & Contribution transition & Change of contextual roles under a representation--realization intervention \\
\bottomrule
\end{tabularx}
\end{table}

At the unrestricted information level, let actions be all conditional predictive distributions and let utility be log score. Then the optimal coalition value is
\begin{equation}
V^*(S)=-H(Y\mid D_S),
\end{equation}
so the contextual marginal reduces to conditional mutual information,
\begin{equation}
\Delta^*(i\mid S)=I(Y;D_i\mid D_S).
\label{eq:cmi2}
\end{equation}
Under restricted realization, however, Lemma~1 and Proposition~1 show that even equality of the entire Shannon subset-information profile need not preserve the optimized coalition utility function or its contextual differences. PID can therefore remain unchanged under an invertible source recoding while the decision-level contribution system changes. Predictive $\mathcal V$-information captures the accessibility mechanism, whereas $\Gamma_g$ measures the induced change in a representation's \emph{marginal role across contexts}. Blackwell comparison concerns ordering of experiments over decision problems; here the statistical information is held fixed by a bijection and the intervention instead probes compatibility with a declared restricted realization mechanism. Cooperative-game attribution is also adjacent but distinct. A Shapley value averages a representation's marginal contribution over contexts \cite{shapley1953}, whereas CFD retains the unaveraged contextual marginal vector and studies how that vector changes under intervention. Averaging can therefore hide the location of a representation--realization mismatch that is visible in $\Gamma_g(i\mid S)$.

These distinctions give the logical chain
\begin{equation}
\begin{aligned}
\text{information equivalence}&\not\Rightarrow\text{ realization equivalence},\\
\text{information equivalence}&\not\Rightarrow\text{ contribution equivalence},\\
\text{behavioral realization equivalence}&\Rightarrow\text{ contribution equivalence}.
\end{aligned}
\label{eq:chain}
\end{equation}
The stability theorem identifies a sufficient quantitative bridge for the second implication: contribution differences vanish as the attainable behavioral sets before and after recoding approach one another.

\section{Realization-Capacity Stress Test}
The finite witness in Lemma~1 also gives a direct capacity test of the proposed mechanism. We evaluated the same one-bit XOR and bijectively recoded supports under increasingly expressive standard realization families. Table~\ref{tab:synthetic} reports the population optimum for affine rules and full-support fitted accuracy for the nonlinear families. The affine family exhibits the predicted $.25$ utility separation. Once the realization family can express both decision boundaries, the gap disappears.

\begin{table}[t]
\caption{Capacity stress test. Both encodings contain exactly one bit of target information.}
\label{tab:synthetic}
\centering
\small
\begin{tabular}{lccc}
\toprule
Realization family & XOR & Recoded & Gap \\
\midrule
Affine (exact optimum) & .75 & 1.00 & .25 \\
Quadratic logistic & 1.00 & 1.00 & 0 \\
RBF-kernel SVM & 1.00 & 1.00 & 0 \\
Decision tree, depth 2 & 1.00 & 1.00 & 0 \\
\bottomrule
\end{tabular}
\end{table}

The information content is fixed while realization capacity changes. For classification accuracy, Theorem~1 applies to discrete prediction actions equipped with the discrete metric; the experiments evaluate the predicted invariance--failure mechanism rather than numerically estimating the theorem's bound. The gap vanishes once the family realizes equivalent decision behavior on both supports. Predictive $\mathcal V$-information diagnoses family-relative accessibility; CFD tracks how the resulting accessibility changes propagate into contextual coalition contributions and restoration.

\section{Context and Capacity on a Coalition Lattice}
A second exact synthetic calculation isolates context dependence itself. Let $D_1,D_2,D_3$ be independent fair bits and consider four deterministic targets under the uniform distribution on the eight states. For every coalition we computed Bayes-optimal accuracy and the exact best affine-threshold accuracy by enumerating projected dichotomies and testing separability. Table~\ref{tab:battery} reports the full-context marginal of $D_3$ after $D_1,D_2$ are already available.

\begin{table}[t]
\caption{Exact full-context contribution under Bayes (B) and affine (A) realization.}
\label{tab:battery}
\centering
\small
\begin{tabular}{lcccc}
\toprule
Target & $V_B(12)$ & $V_B(123)$ & $\Delta_{3,B}(12)$ & $\Delta_{3,A}(12)$ \\
\midrule
Parity & .500 & 1.000 & .500 & .250 \\
Majority & .750 & 1.000 & .250 & .250 \\
OR & .875 & 1.000 & .125 & .125 \\
AND & .875 & 1.000 & .125 & .125 \\
\bottomrule
\end{tabular}
\end{table}

For parity, $D_3$ contributes nothing in the empty or one-companion contexts but contributes $.5$ after both other bits under Bayes realization; the affine family realizes only $.25$ of that full-context opportunity. In contrast, majority, OR, and AND show no Bayes--affine discrepancy in the reported full-context marginal. Thus context can create an informational opportunity while realization capacity determines how much of that opportunity becomes predictive utility. The parity contribution signatures are
\begin{equation}
\mathbf c_{3,B}=(0,0,0,.5),\qquad \mathbf c_{3,A}=(0,0,0,.25),
\end{equation}
with entries indexed by $\varnothing,\{1\},\{2\},\{1,2\}$. A single context-free ``complementarity'' score cannot represent this structure. Together with the recoding stress test, the battery separates two axes of CFD: \emph{where} a representation contributes across the coalition lattice and \emph{how much} of that contribution a realization family can attain.

\section{Controlled Recoding Experiment}
We tested the mechanism using frozen time ($T$), frequency ($F$), and time--frequency ($TF$) ECG embeddings from 928 observations used in prior CFD work \cite{oladunni2025cfd}. Duplicate-linked observations were grouped before splitting into 491 integrity groups; the fixed split contained 648 training, 143 validation, and 137 sealed test observations. No integrity group crossed splits, and the post-split audit found no exact test-to-training duplicates in the constructed or frozen-embedding representations.

After train-fitted standardization, we applied the coordinatewise bijection
\begin{equation}
g(x)=\operatorname{sgn}(x)|x|^3,\qquad
 g^{-1}(z)=\operatorname{sgn}(z)|z|^{1/3}
\end{equation}
to the $TF$ embedding. This transformation preserves all information in $TF$ but changes its geometry for an affine learner. We compared: (i) raw standardized embeddings, (ii) bijectively recoded embeddings under the same one-versus-rest logistic learner, and (iii) recoded embeddings passed through the exact inverse before the same learner. All conditions used identical data, frozen embeddings, standardization, and learner settings; only the $TF$ intervention differed. The inverse condition applied analytical $g^{-1}$ before classification, not a learned decoder. Accuracy-difference uncertainty used 5,000 paired resamples of test integrity groups.

\begin{table}[t]
\caption{Information-preserving recoding and inverse restoration. Intervals are paired 95\% integrity-group bootstrap intervals for recoded minus raw accuracy.}
\label{tab:results}
\centering
\small
\begin{tabular}{lrrrr}
\toprule
Coalition & Raw & Recoded & Restored & Recoding gap \\
\midrule
$TF$       & .7372 & .6277 & .7372 & $-.1095$ \\
$T+TF$     & .6569 & .5839 & .6569 & $-.0730$ \\
$F+TF$     & .7883 & .5912 & .7883 & $-.1971$ \\
$T+F+TF$   & .7372 & .6058 & .7372 & $-.1314$ \\
\bottomrule
\end{tabular}
\vspace{1mm}
\parbox{0.95\columnwidth}{\footnotesize 95\% intervals for the four gaps, respectively: $[-.2391,-.0095]$, $[-.2073,.0357]$, $[-.3631,-.0177]$, and $[-.2837,.0082]$.}
\end{table}

Table~\ref{tab:results} shows the predicted separation. The intervals exclude zero for $TF$ and $F+TF$, despite the recoding being bijective. Applying the declared inverse restores every reported coalition accuracy exactly; the restoration error and its bootstrap interval are $0$ and $[0,0]$ in every context. The cubic map is not proposed as an optimal ECG representation; it is a controlled, exactly invertible intervention that alters linear decision geometry. It therefore holds represented information fixed while perturbing compatibility with the declared realization family, directly testing the mechanism in Theorem~1.

Because the intervention acts only on $TF$, every coalition not containing $TF$ is identical in the raw and recoded systems. Therefore, for each context $S\subseteq\{T,F\}$, the finite-sample analogue of the CFD contribution defect for adding $TF$ simplifies to the absolute recoding response of $S\cup\{TF\}$:
\begin{align}
\widehat\Gamma_g(TF\mid S)
&=\left|\widehat E_g(S\cup\{TF\})-\widehat E_g(S)\right|\nonumber\\
&=\left|\widehat E_g(S\cup\{TF\})\right|.
\label{eq:empGamma}
\end{align}
The second equality is specific to this intervention: because $g$ touches only $TF$, the response of every baseline coalition $S\subseteq\{T,F\}$ is exactly zero. The empirical defect therefore becomes the absolute change in the corresponding coalition that contains $TF$, making Table~\ref{tab:gamma} a derived summary of Table~\ref{tab:results}, not a separate experiment.

Hence the observed finite-sample defects are $.1095$, $.0730$, $.1971$, and $.1314$ for contexts $\varnothing$, $\{T\}$, $\{F\}$, and $\{T,F\}$, respectively. The same information-preserving intervention therefore changes the realized role of $TF$ by different amounts depending on what other representations are available. This is the context-indexed transition that is invisible to a single global accuracy difference.

After applying $g^{-1}$, the learner again receives the original standardized coordinates; the zero restored defect in all four contexts is therefore the finite-sample analogue of exact behavioral closure in Corollary~1.

\section{Discussion and Conclusion}
Information preservation and predictive-contribution preservation are distinct invariance requirements: a bijection can preserve Shannon information for every coalition while changing the best utility attainable by a restricted realization family. CFD makes this distinction explicit by indexing contribution by coalition and realization family. Its compatibility bound shows that $\Gamma_g(i\mid S)$ is controlled by realization incompatibility before and after adding $i$; exact closure forces invariance, whereas approximate closure bounds the change. A nonzero defect therefore need not indicate information loss, but may reveal representation--learner mismatch.

The experiments test the mechanism, not universal complementarity: the formulation is distribution independent, while ECG tests its predicted failure--restoration pattern in a non-synthetic multi-representation system. The synthetic affine gap disappears for richer families; in ECG, recoding changes the time--frequency contribution across all four contexts, while exact inverse restoration returns every defect to zero. This closure--failure--restoration pattern is consistent with the predicted dependence on context and realization family. CFD complements PID, predictive $\mathcal V$-information, Blackwell comparison, and Shapley attribution by retaining coalition-indexed contribution changes under representation intervention. The bound is sufficient; finite-sample error is outside the population theorem, and broader tasks, realization families, and uncertainty guarantees remain future work.

\clearpage

\begingroup
\footnotesize
\setlength{\itemsep}{0pt}

\endgroup
\end{document}